\documentclass{article}
\usepackage{journal2e_v1}

\usepackage{url}
\usepackage{epsf}
\usepackage{epsfig}
\usepackage{amsmath}
\usepackage{subfigure}
\usepackage{amssymb}
\usepackage{algorithm}
\usepackage{multirow}
\usepackage{multicol}
\usepackage{wrapfig}
\usepackage{epstopdf}
\usepackage{makecell}
\usepackage{natbib}
\usepackage{booktabs}
\usepackage{graphicx}
\usepackage{subfigure}
\usepackage{wrapfig}
\usepackage{float}
\usepackage{bbm, dsfont}
\newfloat{figtab}{htb}{fgtb}
\makeatletter
  \newcommand\figcaption{\def\@captype{figure}\caption}
  \newcommand\tabcaption{\def\@captype{table}\caption}
\makeatother

\usepackage{algorithmicx,algorithm}
\usepackage[noend]{algpseudocode}

\def\0{{\bf 0}}
\def\1{{\bf 1}}

\title{Intervention Generative Adversarial Networks}
\author{\name Jiadong Liang\footnotemark[1] \\
\addr Center for Data Science, Peking University \\
\texttt{jdliang@pku.edu.cn}
\AND
\name Liangyu Zhang\footnotemark[1] \\
\addr Center for Data Science, Peking University \\
\texttt{zhangliangyu@pku.edu.cn}
\AND
\name Cheng Zhang\footnotemark[2]\\
\addr School of Mathematical Sciences, Peking University \\
\texttt{chengzhang@math.pku.edu.cn}
\AND
Zhihua Zhang \\
\addr School of Mathematical Sciences, Peking University \\
\texttt{zhzhang@math.pku.edu.cn} \\ 
}
\begin{document}

\maketitle
\renewcommand{\thefootnote}{\fnsymbol{footnote}}
\footnotetext[1]{These authors contributed equally to this work.}
\footnotetext[2]{Corresponding author.}

\begin{abstract}%
In this paper we propose a novel approach for stabilizing the training process of Generative Adversarial Networks as well as alleviating the mode collapse problem.
The main idea is to introduce a regularization term that we call  \textit{intervention loss} into the objective.  We refer to the resulting generative model as \textit{Intervention Generative Adversarial Networks} (IVGAN). 
By perturbing the latent representations of real images obtained from an auxiliary encoder network with Gaussian invariant interventions and penalizing the dissimilarity of the distributions of the resulting generated images, the intervention loss provides more informative gradient for the generator, significantly improving GAN's training stability. 
We demonstrate the effectiveness and efficiency of our methods via solid theoretical analysis and thorough evaluation on standard real-world datasets as well as the stacked MNIST dataset.
\end{abstract}

\section{Introduction}

As one of the most important advances in generative models in recent years, Generative Adversarial Networks (GANs) \citep{DBLP:conf/nips/GoodfellowPMXWOCB14} have been attracting great attention in the machine learning community.
GANs aim to train a generator network that transforms simple vectors of noise to produce ``realistic'' samples from the data distribution.
In the basic training process of GANs, a discriminator and a target generator are trained in an adversarial manner. The discriminator tries to distinguish the generated fake samples from the real ones, and the generator tries to fool the discriminator into believing the generated samples to be real.  

Although successful, there are two major challenges in training GANs: the instability of the training process and the mode collapse.
To deal with these problems, one class of approaches focus on designing more informative objective functions (\citep{salimans2016improved}, \citep{DBLP:journals/corr/MaoLXLW16}, \citep{kodali2018convergence}, \citep{arjovsky1701towards}, \citep{arjovsky2017wasserstein}, \citep{gulrajani2017improved}, \citep{zhou2019lipschitz}).
For example, \citep{DBLP:journals/corr/MaoLXLW16} proposed \emph{Least Squares GAN} (LSGAN) that uses the least squares loss to penalize the outlier point more harshly.
\citep{arjovsky1701towards} discussed the role played by the Jensen-Shannon divergence in the GAN training and suggested to use the Wasserstein distance instead. Accordingly, WGAN\citep{arjovsky2017wasserstein} and WGAN-GP\citep{gulrajani2017improved} have been proposed,
which greatly mitigate the problem of unstable training and mode collapse.
Other approaches enforce proper constraints on latent space representations to better capture the data distribution (\citep{DBLP:journals/corr/MakhzaniSJG15}, \citep{DBLP:journals/corr/LarsenSW15}, \citep{che2016mode}, \citep{tran2018dist}).
A representative work is the \textit{Adversarial Autoencoders} (AAE)\citep{DBLP:journals/corr/MakhzaniSJG15} which uses the discriminator to distinguish the latent representations generated by encoder from Gaussian noise. 
\citep{DBLP:journals/corr/LarsenSW15} employed image representation in the discriminator as the reconstruction basis of a VAE.
Their method turns pixel-wise loss to feature-wise, which can capture the real distribution more simply when some form of invariance is induced.
Different from VAE-GAN, \citep{che2016mode} regarded the encoder as an auxiliary network, which can promote GANs to pay much attention on missing mode and derive an objective function in a form similar to VAE-GAN.

In this paper we propose a novel technique to improve the training of most GAN models as well as the quality of generated images.
The core of our approach is to define a regularization term based on the latent representations of real images generated by an encoder network.
More specifically, we introduce auxiliary intervention operations that preserve the standard Gaussian (e.g., the noise distribution) to these latent representations.
The perturbed latent representations are then fed into the generator to produce \emph{intervened} samples.
We then introduce a classifier network to identify the right intervention operations that would have led to these intervened samples.
The resulting negative cross-entropy loss is added as a regularizer to the objective when training the generator.
We call this regularization term the \textit{intervention loss} and our approach \emph{InterVention Generative Adversarial Nets} (IVGAN).

We theoretically prove that the intervention loss is equivalent with the JS-divergence among multiple intervened distributions.
Most importantly, these intervened distributions interpolate between the original generative distribution of GAN and the data distribution, allowing useful information for the generator that is previously unavailable in original GAN models (see a thorough analysis on a toy example in Example \ref{exm:1}).
We show empirically that our model can be trained efficiently by utilizing the parameter sharing strategy between the discriminator and the classifier.
The models trained on the MNIST, CIFAR-10, LSUN and STL-10 datasets successfully generate diverse, visually appealing objects,  outperforming state-of-the-art baseline methods such as WGAN-GP in terms of the \textit{Fr\`echet Inception Distance} (FID) (proposed in \citep{heusel2017gans}).
We also perform a series of experiments on the stacked MNIST dataset and the results show that our proposed method can also effectively alleviate the mode collapse problem. Moreover, an ablation study is conducted, which validates the effectiveness of the proposed intervention loss.

In summary, our work offers three major contributions as follows. (\romannumeral1) We propose a novel method that can improve GAN's training as well as generating performance. (\romannumeral2) We theoretically analyze our proposed model and give insights on how it makes the gradient of generator more informative and thus stabilizes GAN's training. (\romannumeral3) We evaluate the performance of our method on both standard real-world datasets and the stacked MNIST dataset by carefully designed expriments, showing that our approach is able to stabilize GAN's training as well as improve the quality and diversity of generated samples.

\section{Preliminaries}
\label{gen_inst}
\paragraph{Generative adversarial nets}
The basic idea of GAN is to utilize a discriminator to continuously push a generator to map Gaussian noise to samples drawn according to an implicit data distribution. The objective function of the vanilla GAN takes the following form:
\begin{equation}
\min\limits_{G}\max\limits_{D} \Big\{V(D,G) \triangleq \mathbb{E}_{x\sim p_{data}}\log (D(x))+\mathbb{E}_{z\sim p_z}\log (1-D(G(z)))\Big\}, \label{vanilla GAN loss}
\end{equation}
where $p_z$ is a prior distribution (e.g., the standard Gaussian). It can be easily seen that when the discriminator reaches its optimum, that is, $D^\ast(x) = \frac{p_{data}(x)}{p_{data}(x)+p_G(x)}$, the objective is equivalent to the Jensen-Shannon (JS) divergence between the generated distribution $p_G$ and data distribution
$p_{data}$:
\[
{JS}(p_G\|p_{data}) \triangleq \frac{1}{2}\left\{{KL}(p_G\|\frac{p_G+p_{data}}{2})+ {KL}(p_{data}\|\frac{p_G+p_{data}}{2})\right\}.
\]
Minimizing this JS divergence guarantees that the generated distribution converges to the data distribution given adequate model capacity.

\paragraph{Multi-distribution JS divergence}
The JS divergence between two distributions $p_1$ and $p_2$ can be rewritten as
\[
    {JS}(p_1\|p_2) =H(\frac{p_1+p_2}{2})-\frac{1}{2}H(p_1)-\frac{1}{2}H(p_2),
\]
where $H(p)$ denotes the entropy of distribution $p$. We observe that the JS-divergence can be interpreted as the entropy of the mean of the two distribution minus the mean of two distribution's entropy. So it is immediate to generalize the JS-divergence to the setting of multiple distributions. In particular, we define the JS-divergence of $p_1, p_2, \dots,p_n$ with respect to weights $\pi_1, \pi_2, \dots,\pi_n$  ($\sum \pi_i=1$ and $\pi_i\ge 0$) as
\begin{equation}
    {JS}_{\pi_1,\dots,\pi_n}(p_1,p_2, \dots,p_n) \triangleq H(\sum\limits_{i=1}\limits^{n}\pi_ip_i)-\sum\limits_{i=1}\limits^n\pi_iH(p_i).
\end{equation}
The two-distribution case described above is actually a special case of the `multi-JS divergence', where  $\pi_1=\pi_2=\frac{1}{2}$.
When $\pi_i > 0 \;\forall i$, it can be found immediately by Jensen's inequality that ${JS}_{\pi_1,\dots,\pi_n}(p_1,p_2,\dots,p_n)=0$ if and only if $p_1=p_2=\dots=p_n$.

\section{Methodology}
\label{headings}
Training GAN has been challenging, especially when the generated distribution and the data distribution are far away from each other.
In such cases, the discriminator often struggles to provide useful information for the generator, leading to instability and mode collapse problems.
The key idea behind our approach is that we construct auxiliary intermediate distributions that interpolate between the generated distribution and the data distribution.
To do that, we first introduce an encoder network and combine it with the generator to learn the latent representation of real images within the framework of a standard autoencoder.
We then perturb these latent representations with carefully designed intervention operations before feeding them into the generator to create these auxiliary interpolating distributions.
A classifier is used to distinguish the intervened samples which leads to an intervention loss that penalizes the dissimilarity of these intervened distributions. 
The reconstruction loss and the intervention loss are added as regularization terms to the standard GAN loss for training.
We start with an introduction of some  notation and definitions.

\begin{definition}[Intervention]
Let $O$ be a transformation on the space of d-dimension random vectors and $\mathbbm{P}$ be a probability distribution whose support is in $\mathbb{R}^d$. We call $O$ a $\mathbbm{P}$-intervention if for any d-dimensional random variable $X$, $X\sim \mathbbm{P}\Rightarrow O(X)\sim \mathbbm{P}$. 
\end{definition}
Since the noise distribution in GAN models is usually taken to be standard Gaussian, we use the standard Gaussian distribution as the default choice of $\mathbbm{P}$ and abbreviate the $\mathbbm{P}$-intervention as \textit{intervention}, unless otherwise claimed.
To make the invariant distribution identifiable, we need a \emph{complete} group of interventions which is defined as follows.  
\begin{definition}[Complete Intervention Group]
Suppose $S=\left\{O_1,O_2,\ldots,O_k\right\}$ is a group of interventions. We say that $S$ is complete for distribution $\mathbbm{P}$ if
\[
\mathbbm{Q} \text{ is a distribution s.t. } O \text{ is a } \mathbbm{Q}\text{-intervention}, \forall O \in S \Longleftrightarrow\mathbbm{Q}=\mathbbm{P}.
\]
\end{definition}
One of the simplest complete groups of interventions is \textbf{block substitution}. Let $Z\in\mathbbm{R}^d$ be a random variable, $k\in \mathbbm{N}$ and $k|d$.
We slice $Z$ into $k$ blocks so that every block is in $\mathbbm{R}^{\frac{d}{k}}$.
A block substitution intervention $O_i$ is to replace the $i$th block of $Z$ with Gaussian noise, $i=1,\ldots, \frac{d}{k}$.
We will use block substitution interventions in the rest of the paper unless otherwise specified.
Note that our theoretical analysis as well as the algorithm framework does not depend on the specific choice of the intervention group, as long as it is complete.

\paragraph{Notation} We use $E, G, D, f$ to represent encoder, generator, discriminator and classifier, respectively. $p_{real}$ means the distribution of the real data, and $p_z$ is the prior distribution of noise $z$ defined on the latent space (usually is taken to be Gaussian). Let $O_i, i=1, \dots, k$, denote $k$ different interventions, and $X_i$ be the intervened sample created from $O_i$ with distribution $p_i$.
\label{sec:lapexp}

\begin{figure}
    \centering
    \includegraphics[width=1.\linewidth]{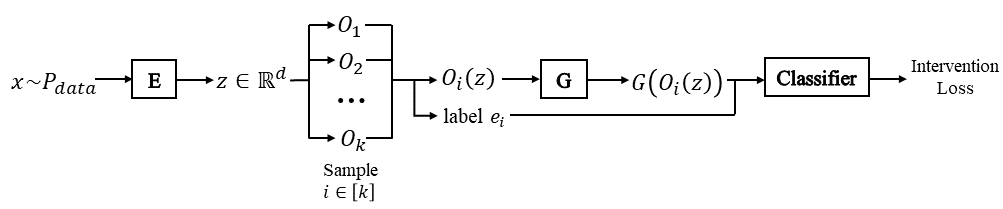}
    \caption{The figure shows the specific process of generating $X^\prime$ with labels as the input of the classifier. In the picture, $Z$ represent the latent code of given image $X$. Intervened samples are then generated through different intervention operations.}
    \label{fig:IVloss}
\end{figure}

\paragraph{Intervention loss}
The intervention loss is the core of our approach.
More specifically, given a latent representation $z$ that is generated by an encoder network $E$, we sample an intervention $O_i$ from a complete group $S=\left\{O_1, \dots, O_k\right\}$ and obtain the corresponding intervened latent variable $O_i(z)$ with label $e_i$.
These perturbed latent representations are then fed into the generator to produce \emph{intervened} samples.
We then introduce an auxiliary classifier network to identify which intervention operations may lead to these intervened samples.
The intervention loss  $ \mathcal{L}_{IV}(G,E)$ is simply the resulting negative cross-entropy loss and we add that as a regularizer to the objective function when training the generator.
As we can see,  the intervention loss is used to penalize the dissimilarity of the distributions of the images generated by different intervention operations. 
Moreover, it can be noticed that the classifier and the combination of the generator and the encoder are playing a two-player adversarial game and we will train them in an adversarial manner.  In particular, we define
\begin{equation}
  \mathcal{L}_{IV}(G,E) = -\min\limits_f V_{class}, \quad \mbox{where }  \quad V_{class}=\mathbbm{E}_{i\sim \mathcal{U}([k])}\mathbbm{E}_{ x^\prime\sim p_i}-e_{i}^\mathrm{T}\log f(x^\prime). 
\end{equation}\label{intervention  loss} \label{optimal intervention loss}

\begin{theorem}[Optimal Classifier]
The optimal solution of the classifier is the conditional probability of label $y$ given $X^\prime$, where $X^\prime$ is the intervened sample generated by the intervention operation sampled from $S$. And the minimum of the cross entropy loss is equivalent with the negative of the Jensen Shannon divergence among the $\left\{p_1,p_2,...,p_k\right\}$. That is, 
\begin{equation}
    f^*_i(x) = \frac{p_i(x)}{\sum\nolimits_{j=1}^k p_j(x)} \quad \mbox{and} \quad  
  \mathcal{L}_{IV}(G,E) = {JS}(p_1,p_2, ..., p_k) + \mbox{Const}.
 %   \min\limits_{f}V_{class} = -{JS}(p_1,p_2, ..., p_k) + \mbox{Const}.
\end{equation}
\end{theorem}
\begin{proof}
The conditional probability of $X^\prime$ given label can be written as $\mathbbm{P}(X^\prime| e_i) = p_i(X^\prime)$, so further $\mathbbm{P}(X^\prime,e_i) = \frac{1}{k}p_i$. And we denote the marginal distribution of $x$ as $p(x)=\frac{1}{k}\sum\limits_{i=1}^k p_i(x)$. Cause the activation function at the output layer of the classifier is softmax, we can rewrite the loss function into a more explicit form:
\begin{equation*}
\begin{aligned}
    V_{class}(f) &= \mathbbm{E}_{i\sim\mathcal{U}[k]}\mathbbm{E}_{x^\prime\sim p_i}-e_i^\mathrm{T}\log f(x^\prime)=\mathbbm{E}_{i\sim\mathcal{U}[k]}\mathbbm{E}_{x^\prime\sim p_i}-\log f_i(x)\\ &= \frac{1}{k}\int \sum\limits_{i=1}\limits^k -p_i(x)\log f_i(x)dx=\int p(x)\left\{-\sum\limits_{i=1}\limits^k p(e_i|x)\log f_i(x)\right\}dx.
\end{aligned}
\end{equation*}
Let $g_i(x) = \frac{f_i(x)}{p(e_i|x)}$, then $\sum\limits_{i=1}\limits^k p(e_i|x)g_i(x)=1$. And notice that $\sum\limits_{i=1}\limits^k p(e_i|x)=1$. By Jensen's inequality, we have:
\begin{equation*}
    \begin{aligned}
    &\sum\limits_{i=1}\limits^k-p(e_i|x)\log f_i(x)=  \sum\limits_{i=1}\limits^k-p(e_i|x)\log [g_i(x)p(e_i|x)]\\
    & = \sum\limits_{i=1}\limits^k-p(e_i|x)\log g_i(x) + H(p(\cdot|x))
    \ge \log\sum\limits_{i=1}^k p(e_i|x)g_i(x) + H(p_i(\cdot|x))\\
    &= \log 1 + H(p(\cdot|x)) = H(p(\cdot|x)).
    \end{aligned}
\end{equation*}
And $V_{class}(f^*) = \int p(x)H(p_i(\cdot|x))dx$ if and only if $g^*_i(x)= g^*_j(x)$ for any $i\neq j$, which means that $\frac{f^*_i(x)}{p(e_i|x)}=r \quad \forall i\in [k]$, where $r\in\mathbbm{R}$. Notice that $\sum\limits_{i=1}^k f^*_i(x)=1$, it is not difficult to get that $f^*_i(x)=p(e_i|x)$. The loss function becomes
\begin{equation}
    \begin{aligned}
    &\frac{1}{k}\int\sum\limits_{i=1}^k -p_i(x)\log p(e_i|x)dx = -H(x)+\sum\limits_{i=1}\limits^k \frac{1}{k}H(p_i) + \log k\\
    &= -{JS}(p_1,p_2, ...,p_k) + \log k
    \end{aligned}
\end{equation}
\end{proof}

Clearly,  the intervention loss is an approximation of the JS divergence among the intervened distributions $\left\{p_i:i\in [k]\right\}$. 
If the intervention loss reaches its global minimum, we have $p_1=p_2=\dots=p_k$. 
And it reaches the maximum $\log k$ if and only if the supports of these $k$ distributions do not intersect with each other.
This way, the probability that the `multi' JS-divergence has constant value is much smaller, which means the phenomenon of gradient vanishing should be rare in IVGAN.
Moreover, as shown in the following example, due to these auxiliary intervened distributions, the intervention loss is likely to provide more informative gradient for the generator that is not previously available in other GAN variants.\\[-10pt]

\begin{example}[Square fitting] \label{exm:1}
Let $X_0$ be a random variable with distribution $\mathcal{U}(\alpha)$, where $\alpha = [-\frac{1}{2}, \frac{1}{2}]\times [-\frac{1}{2}, \frac{1}{2}]$. And $X_1 \sim \mathcal{U}(\beta)$, where $\beta = [a-\frac{1}{2},a+\frac{1}{2}]\times[\frac{1}{2},\frac{3}{2}]$ and $0\le a\le 1$. Assuming we have a perfect discriminator (or classifier), we compute the vanilla GAN loss (i.e. the JS-divergence) and the intervention loss between these two distributions, respectively,
\begin{itemize}
    \item ${JS}(X_0\|X_1) = \log 2$.
    \item
    In order to compute the intervention loss we need figure out two intervened samples' distributions evolved from $\mathcal{U}(\alpha)$ and $\mathcal{U}(\beta)$. $Y_1\sim \mathcal{U}(\gamma_1);\quad \gamma_1 = [-\frac{1}{2},\frac{1}{2}]\times[\frac{1}{2},\frac{3}{2}]$ and $Y_2\sim\mathcal{U}(\gamma_2);\quad \gamma_2 = [a-\frac{1}{2},a+\frac{1}{2}]\times[-\frac{1}{2},\frac{1}{2}]$. Then the intervention loss is the multi JS-divergence among these four distributions:
    \begin{equation*}
        \begin{aligned}
          \mathcal{L}_{IV} & = {JS}(X_0;X_1;Y_1;Y_2)\\
                &= {-}\int_{A^c}\frac{1}{4}\log\frac{1}{4}d\mu {-} \int_A\frac{1}{2}\log\frac{1}{2}d\mu {-} H(X_0) = \frac{\log 2}{2}[\mu(A^c) {+} \mu(A)]\\
                &= \frac{\log 2}{2}\times 2(2-a)-H(X_0)=-(\log2)a-\mbox{Const}.
        \end{aligned}
    \end{equation*}
\end{itemize}
Here $A$ is the shaded part in Figure \ref{fig:square_fitting} and $A^c=\left\{\alpha\cup\beta\cup\gamma_1\cup\gamma_2\right\}\backslash A$. The most important observation is that the intervention loss is a function of parameter $a$ and the traditional GAN loss is always  constant. When we replace the JS with other $f$-divergence, the metric between $\mathcal{U}(\alpha)$ and $\mathcal{U}(\beta)$ would still remain constant. Hence in this situation, we can not get any information from the standard JS for training of the generator but the intervention loss works well.
\end{example}

\begin{figure}
    \centering
    \includegraphics[width=.5\linewidth]{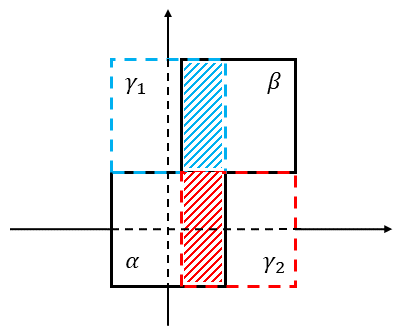}
    \caption{The supports of the two original distribution are the squares with black border, and the supports of the synthetic distributions are the area enclosed by red and blue dotted line, respectively.}
    \label{fig:square_fitting}
\end{figure}

\paragraph{Reconstruction loss}
In some sense we expect our encoder to be a reverse function of the generator. So it is necessary for the objective function to have a term to push the map composed of the Encoder and the Generator to have the ability to reconstruct the real samples. Not only that, we also hope that the representation can be reconstructed from samples in the pixel space. 

Formally, the reconstruction loss can be defined by the $\ell_p$-norm ($p\ge 1$) between the two samples, or in the from of the Wasserstein distance between samples if images are regarded as a histogram. Here we choose to use the $\ell_1$-norm as the reconstruction loss:
\begin{equation}
%\begin{aligned}
    \mathcal{L}_{recon} = \mathbbm{E}_{X\sim p_{real}}\|G(E(X)) {-} X\|_1 +\mathbbm{E}_{i\sim\mathcal{U}([k])}\mathbbm{E}_{x,z\sim p_{real},p_z}\|E(G(O_i(z))) {-} O_i(z)\|_1.
%\end{aligned}
\end{equation}\label{reconstruction loss}

\begin{theorem}[Inverse Distribution]
Suppose the cumulative distribution function of $O_i(z)$ is $q_i$. For any given positive real number $\epsilon$, there exist a $\delta>0$ such that if $\mathcal{L}_{recon} + \mathcal{L}_{IV}\le \delta$ , then $\forall i,j\in [k]$, $\sup\limits_r\|q_i(r)-q_j(r)\|\le \epsilon$.
\end{theorem}
\begin{proof}
According to Theorem 1, for a given real number $\epsilon_1$, we can find another $\delta_1$, when intervention loss is less than $\delta_1$, the distance between $p_i$ and $p_j$ under the measurement of JS-divergence is less than $\epsilon_1$. And because JS-divergence and Total Variance distance (TV) are equivalent in the sense of convergence. So we can bound the TV-distance between $p_i$ and $p_j$ by their JS-divergence. Which means that $\int |p_i -p_j|dx \le \epsilon_0$ when the intervention loss is less than $\epsilon_1$ (we can according to the $\epsilon_0$ to finding the appropriate $\epsilon_1$). Using this conclusion we can deduce $|P(E(G(O_i(z)))\le r)-P(E(G(O_j(z)))\le r)|\le \epsilon_0$, where $r$ is an arbitrary vector in $\mathbbm{R}^d$. Further, we have:
\begin{equation}
\begin{aligned}
    &|P(O_i(z)\le r)-P(O_j(z)\le r)|\le |P(O_i(z)\le r; \|O_i(z) - E(G(O_i(z)))\|> \delta)|\\
    &+ |P(O_j(z)\le r; \|O_j(z) - E(G(O_j(z)))\|> \delta)| + |P(O_i(z)\le r;  \|O_i(z)-E(G(O_i(z)))\|\le \delta)\\
    &-P(O_j(z)\le r; \|O_j(z)-E(G(O_j(z)))\|\le \delta)|
\end{aligned}
\end{equation}\label{weak convergence control}
We control the three terms on the right side of the inequality sign respectively.
\begin{equation}
\begin{aligned}
    &P(O_i(z)\le r;\|O_i(z)-E(G(O_i(z)))\|>\delta)\\
    &\le P(\|O_i(z)-E(G(O_i(z)))\|>\delta)\le \frac{\mathbbm{E}\|O_i(z)-E(G(O_i(z)))\|}{\delta}
\end{aligned}
\end{equation}
And the last term can be bounded by the reconstruction loss. The same trick can be used on $P(O_j(z)\le r; \|O_j(z)-E(G(O_j(z)))\|>\delta)$. Moreover, we have
\begin{equation}
\begin{aligned}
    &P(E(G(O_i(z)))\le r-\delta) - P(\|O_i(z)-E(G(O_i(z)))\|>\delta)\\\le &P(O_i(z)\le r;\|O_i(z)-E(G(O_i(z)))\|\le\delta)\le P(E(G(O_i(z)))\le r+\delta)
\end{aligned}
\end{equation}
Notice that $\lim\limits_{\delta \to 0}P(E(G(O_i(z)))\le r\pm \delta)=P(E(G(O_i(z)))\le r)$. Let $s_i(r,\delta) = |P(E(G(O_i(z)))\le r\pm\delta))-P(E(G(O_i(z)))\le r)|$ then the last term of inequality\ref{weak convergence control} can be bounded as:
\begin{equation}
    \begin{aligned}
    |P(O_i(z)\le r;&  \|O_i(z)-E(G(O_i(z)))\|\le \delta)-P(O_j(z)\le r; \|O_j(z)-E(G(O_j(z)))\|\le \delta)|\\
    \le &|P(E(G(O_i(z)))\le r)-P(E(G(O_j(z)))\le r)|+ P(\|O_i(z)-E(G(O_i(z)))\|>\delta)\\
    &+ s_i(r,\delta) + s_j(r,\delta)
    \end{aligned}
\end{equation}
Every term on the right hand of the inequality can be controlled close to 0 by the inequalities mentioned above
\end{proof}

\paragraph{Adversarial loss}
% \subsection{Adversarial Loss}
%We use the same adversarial loss as in general GAN models.
%It is worth mentioning  that since our method simply adds new terms to the objective of a GAN model,  it is actually free to choose any adversarial loss functions as our adversarial loss, e.g., the binary cross entropy loss and the least square loss.
The intervention loss and reconstruction loss can be added as regularization terms to the adversarial loss in many GAN models, e.g., the binary cross entropy loss in vanilla GAN and the least square loss in LSGAN.
In the experiments, we use LSGAN\citep{DBLP:journals/corr/MaoLXLW16} and DCGAN\citep{radford2015unsupervised} as our base models, and name the resulting IVGAN models IVLSGAN and IVDCGAN respectively.

Now that we have introduced the essential components in the objective of IVGAN, we can write the loss function of the entire model:
\begin{equation}\label{entire loss}
    \mathcal{L}_{model}=\mathcal{L}_{Adv}+\lambda\mathcal{L}_{recon}+\mu\mathcal{L}_{IV},
\end{equation}
where $\lambda$ and $\mu$ are the regularization coefficients for the reconstruction loss and the intervention loss respectively. We summarize the training procedure in Algorithm~\ref{alg:1}. Full workflow of our framework can be seen in Figire \ref{full_workflow}.
\begin{figure}
    \centering
    \includegraphics[width=0.8\textwidth]{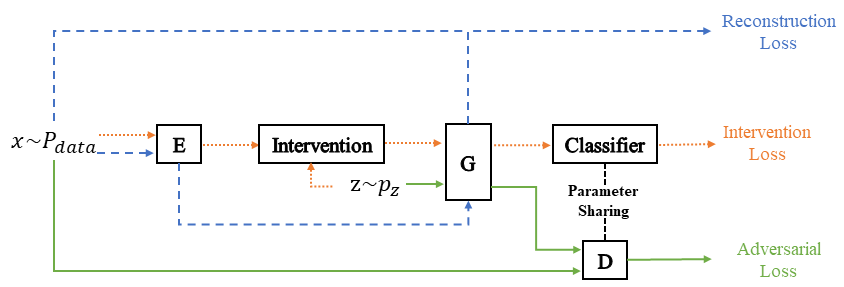}
    \caption{Full workflow of our approach.}
    \label{full_workflow}
\end{figure}

\begin{algorithm}[t]
\caption{Intervention  GAN} \label{alg:1}
\hspace*{0.02in}{\textbf{Input}} 
learning rate $\alpha$, regularization parameters $\lambda$ and $\mu$, dimension $d$  of latent space,  number $k$ of blocks in which the hidden space is divided, minibatch size $n$.
\begin{algorithmic}[1]
\For{number of training iterations}
    \State Sample minibatch $z_j$, $j=1,...,n$, $z_j\sim p_z$
    \State Sample minibatch $x_j$, $j=1,...,n$, $x_j\sim p_{real}$
    \For{number of inner iteration}
        \State $w_j \gets E(x_j)$, $j=1,...,n$
        \State Sample Gaussian noise $\epsilon$
        \State Sample $i_j\in [k]$, $j=1,...,n$
        % \State $M_j \gets \sum\limits_{r=1}\limits^\frac{d}{k}e_{\frac{d}{k}i_j + r}$, $j=1,...,n$
        % \State $w^\prime_j \gets M_j * \epsilon + (1 - M_j) * w_j$
        % \State $x^\prime_j \gets G(w^\prime_j)$
        \State $x^\prime_j \gets G(O_{i_j}(w_j))$
        \State Update the parameters of $D$ by:
        \State $\quad\theta_D \gets \theta_D - \frac{\alpha}{2n}\nabla_{\theta_D}\mathcal{L}_{adv}(\theta_D)$
        \State Update the parameters of $f$ by:
        \State $\quad\theta_f\gets \theta_f + \frac{\alpha}{n}\nabla_{\theta_f}\sum\limits_{j=1}\limits^n\log f_{i_j}(x^\prime_j)$
        \State Calculate $\mathcal{L}_{Adv}$ and $\mathcal{L}_{IV}$
    \EndFor
    \State Update the parameter of $G$ by:
    \State $\quad \theta_G\gets\theta_G + \frac{\alpha}{n}\nabla_{\theta_G}\left\{\hat{\mathcal{L}}_{Adv} + \lambda\hat{\mathcal{L}}_{recon} + \mu\hat{\mathcal{L}}_{IV}\right\}$
    \State Update the parameter of $E$ by:
    \State $\quad \theta_E\gets\theta_E + \frac{\alpha}{n}\nabla_{\theta_E}\left\{\lambda\hat{\mathcal{L}}_{recon}+\mu\hat{\mathcal{L}}_{IV}\right\}$
\EndFor
\end{algorithmic}
\end{algorithm}

\section{Related Work}
In order to address GAN's unstable training and mode missing problems, many researchers have turned their attention to the latent representations of samples. \citep{DBLP:journals/corr/MakhzaniSJG15} propose the \textit{Adversarial Autoencoder} (AAE). As its name suggests, AAE is essentially a probabilistic autoencoder based on the framework of GANs. Unlike classical GAN models, in the setting of AAE the discriminator's task is to distinguish the latent representations of real images that are generated by an Encoder network from Gaussian noise. And the generator and the encoder are trained to fool the discriminator as well as reconstruct the input image from the encoded representations. However, the generator can only be trained by fitting the reverse of the encoder and cannot get any information from the latent representation.

The VAE-GAN\citep{DBLP:journals/corr/LarsenSW15} combines the objective function from a VAE model with a GAN and utilizes the learned features in the discriminator for better image similarity metrics, which is of great help for the sample visual fidelity. Considering the opposite perspective, \citep{che2016mode} claim that the whole learning process of a generative model can be divided into the manifold learning phase and the diffusion learning phase. And the former one is considered to be the source of the mode missing problem. They propose \textit{Mode Regularized Generative Adversarial Nets} which introduce a reconstruction loss term to the training target of GAN to penalize the missing modes. It is shown that it actually ameliorates GAN's 'mode missing'-prone weakness to some extent. However, both of them fail to fully excavate the impact of the interaction between VAEs and GANs.

\citep{kim2018disentangling} propose Factor VAE where a regularization term called total correlation penalty is added to the traditional VAE loss. The total correlation is essentially the Kullback-Leibler divergence between the joint distribution $p(z_1,z_2,\dots,z_d)$ and the product of marginal distribution $p(z_i)$. Because the closed forms of these two distribution are unavailable, Factor VAE uses adversarial training to approximate the likelihood ratio. 

\section{Experiments}
In this section we conduct a series of experiments to study IVGAN from multiple aspects. First we evaluate IVGAN's performance on standard real-world datasets, including MNIST \citep{mnist}, CIFAR10 \citep{cifar10}, LSUN \citep{lsun} and STL-10 \citep{stl}. Then we show IVGAN's ability to tackle the mode collapse problem on the stacked MNIST dataset. Finally, through an ablation study we investigate the performance of our proposed method under different settings of hyperparameters and demonstrate the effectiveness of the intervention loss. As stated in the previous sections, we employ the block substitution intervention when implementing our method in all the experiments.

We implement our models using PyTorch \citep{pytorch}. In order to make fair comparison, we try to minimize the choices of the architectures, hyperparameters and optimizer settings in our implementation of IVGAN as well as other baseline methods. (see the specific network architectures in Table \ref{arch})
The classifier we use to compute the intervention loss shares the parameters with the discriminator except for the output layer.
Accordingly, all input images are resized to have $64\times64$ pixels.
We use 100-dimensional standard Gaussian distribution as the prior $p_z$. The hyperparameters are set as follows: $k=4$; $\lambda=0.25$, $\mu=0.5$ when training G \& D and $\lambda=1$, $\mu=1$ when training E. Our models are trained with the Adam optimizer. The learning rate is set to $1\times 10^{-4}$ when training the discriminator, the classifier and the generator, and $5\times 10^{-3}$ when training the encoder. The momentum coefficient $\beta_1$ is set to $0.5$ and $\beta_2$ is set to $0.999$. In addition, in the very early stage of training we add zero-mean Gaussian noise to the pixel space of both generated images and images sampled from the ground truth distribution. And the variance of the added noise is decayed as the training proceeds.

As for the baseline methods the architectures of the discriminator and the generator are identical to our method.
We try multiple optimizer settings and hyperparameters which are already proved to work well and choose the best ones from them. We also try adding noise to the input images, and find that such tricks can only worsen the performance of the baseline methods.

\begin{table}
  \caption{The NN architecture used by us, where CONV denotes the convolutional layer, TCONV denotes the transposed convolutional layer, FC denotes the fully-connected layer, BN denotes the batch normalization layer, and (K4, S1, O512) denotes a layer with kernel of size 4, stride 1, and 512 output channels.}
  \label{arch}
  \centering
  \begin{tabular}{ccc}
    \toprule
    D   & G     & E \\
    \midrule
    INPUT 64$\times$64$\times$3& INPUT z & INPUT 64$\times$64$\times$3 \\
    CONV(K4, S2, O64) & TCONV(K4, S1, O512) & CONV(K4, S2, O64)\\
    BN, LeakyReLU & BN, ReLU & LeakyReLU \\
    CONV(K4, S2, O128)& TCONV(K4, S2, O256)& CONV(K4, S2, O128)\\
    BN, LeakyReLU & BN, ReLU & BN, LeakyReLU \\
    CONV(K4, S2, O256) & TCONV(K4, S2, O128) & CONV(K4, S2, O256)\\
    BN, LeakyReLU & BN, ReLU & BN, LeakyReLU \\
    CONV(K4, S2, O512) & TCONV(K4, S2, O64)& CONV(K4, S2, O512)\\
    BN, LeakyReLU & BN, ReLU & BN, LeakyReLU \\
    FC(O1) & TCONV(K4, S2, O3) & CONV(K4, S2, O100)\\
    LOSS & Tanh & BN\\
    \bottomrule
  \end{tabular}
\end{table}
\paragraph{Real-world datasets experiments}
We train our proposed method on four standard real-world datasets to investigate its training stability and quality of the generated images. The datasets we use include CIFAR-10, MNIST, STL-10, and a subclass named ``church\_outdoor" of the LSUN dataset. We use the \textit{Fr\`echet Inception Distance} (FID) (proposed in \citep{heusel2017gans}) to measure the performance of our models as well as other baseline methods in a quatitative way. 

The FID results are listed in Table \ref{minFID}, and the training curves of the baseline methods and IVGAN on four different datasets are shown in Figure \ref{FID curve}. We see that on each datasets, IVGAN or IVLSGAN obtain better FID scores than the baselines. Moreover, the figure of training curves also suggests the learning process of IVGAN and IVLSGAN is smoother and steadier compared to DCGAN, LSGAN or MRGAN \citep{che2016mode}, and converges much faster than WGAN or WGAN-GP. Samples of generated images on all datasets are included in Figure \ref{fig:samples}.

\begin{table}
  \caption{Minimum of FIDs on different Datasets. The FID results are calculated every 10 epochs. Lower is better .}
  \label{minFID}
  \vspace{0.5cm}
  \centering
  \begin{tabular}{ccccc}
    \toprule
    Methods & MNIST & CIFAR10 & LSUN (Church\_outdoor) & STL-10 \\
    \midrule
    DCGAN & 10.7 & 51.2 & 23.6 & 44.0\\
    LSGAN & 10.9 & 34.7 & 30.9 & 60.1\\
    WGAN & 12.0 & 34.6 & 33.3 & 55.0\\
    WGAN-GP & 11.0 & 38.3 & 26.8 & 47.9\\
    MRGAN & 7.4 & 33.0 & 22.4 & 45.2\\
    \midrule
    IVDCGAN & \textbf{5.4} & 32.2 & 20.5 & 43.7\\
    IVLSGAN & 8.3 & \textbf{28.2} & \textbf{18.9} & \textbf{41.6}\\
    \bottomrule
  \end{tabular}
\end{table}

\begin{figure}[ht]
\begin{subfigure}
  \centering
  % include first image
  \includegraphics[width=.5\linewidth]{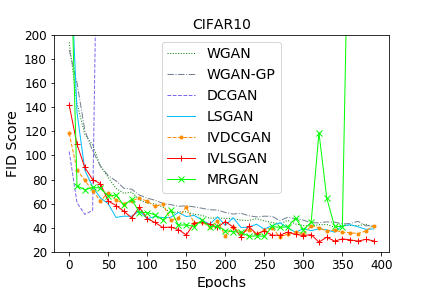}  
\end{subfigure}
\begin{subfigure}
  \centering
  % include second image
  \includegraphics[width=.5\linewidth]{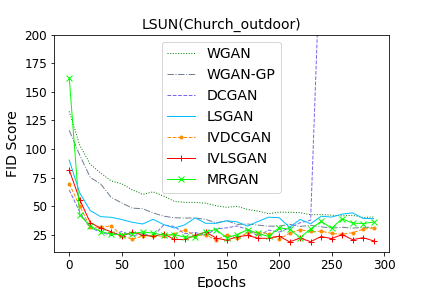}  
\end{subfigure}
\caption{Training curves of different methods in terms of FID on different datasets. {\bf Left}: CIFAR10. {\bf Right}: Church Outdoors. Note that on both datasets the training of DCGAN fails at some point. The raise of curves in the later stage may indicate mode collapse.}
\label{FID curve}
\end{figure}

\begin{figure}[ht]
\begin{subfigure}
 \centering
 % include first image
 \includegraphics[width=.24\linewidth]{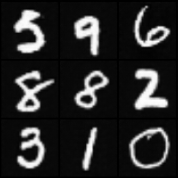}  
\end{subfigure}
\begin{subfigure}
 \centering
 % include first image
 \includegraphics[width=.24\linewidth]{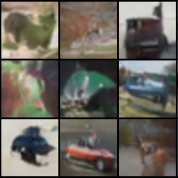}  
\end{subfigure}
\begin{subfigure}
 \centering
 % include first image
 \includegraphics[width=.24\linewidth]{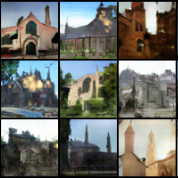}  
\end{subfigure}
\begin{subfigure}
 \centering
 % include first image
 \includegraphics[width=.24\linewidth]{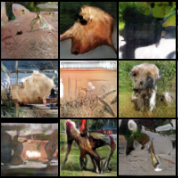}  
\end{subfigure}
\caption{Samples of generated images. The samples are not cherry-picked, but generated in a random way.}
\label{fig:samples}
\end{figure}

\paragraph{Stacked MNIST experiments}
The original MNIST dataset contains 70K images of $28\times28$ handwritten digits. Following the same approaches in \citep{unrolledGAN}, \citep{veegan}, \citep{PACGAN}, we increase the number of modes of the dataset from 10 to $1000=10\times10\times10$ by stacking the images. Specifically, the new stacked MNIST dataset consists of images which is constructed by stacking three random MNIST images into a $28\times28\times3$ RGB image. The metric we use to evaluate a model's robustness to mode collapse problem is the number of modes captured by the model, as well as the KL divergence between the generated distribution over modes and the expected (uniform) one. After an image is generated, we determine which of the 1000 modes the generated image belong to by feeding each of the three channels to a pre-trained MNIST digit classifier.

Our result are shown in Table \ref{stackMNIST}. It can be seen that our model works very well to prevent the mode collapse problem. Both IVLSGAN and IVDCGAN are able to reach all 1,000 modes and greatly outperforms early approaches to mitigate mode collapse, such as VEEGAN \citep{veegan}, and Unrolled GAN \citep{unrolledGAN}. Moreover, the performance of our model is also comparable to method that is proposed more recently, such as the PacDCGAN \citep{PACGAN}. Figure \ref{stackMNISTimage} shows images generated randomly by our model as well as the baseline methods.

\begin{table}
  \caption{Results of our stacked MNIST experiments. The first four rows are directly copied from \cite{PACGAN} and \cite{veegan}. And the last three rows are obtained after training each model for 100K iterations, respectively.}
  \label{stackMNIST}
  \centering
  \begin{tabular}{ccc}
    \toprule
       & Modes     & KL Divergence \\
    \midrule
    DCGAN & 78.9 & 4.50 \\
    VEEGAN & 150.0 & 2.95\\
    Unrolled GAN & 48.7 & 4.32 \\
    PacDCGAN & 1000 & 0.06\\
    \midrule
    LSGAN & 53 & 3.88\\
    IVLSGAN & 1000 & 0.07\\
    IVDCGAN & 1000 & 0.08 \\
    \bottomrule
  \end{tabular}
\end{table}
 
\begin{figure}[ht]
\begin{subfigure}
 \centering
 \includegraphics[width=.32\linewidth]{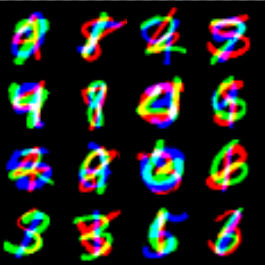}  
\end{subfigure}
\begin{subfigure}
 \centering
 \includegraphics[width=.32\linewidth]{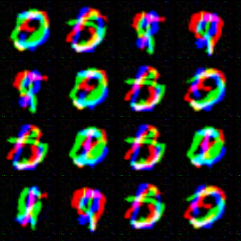}  
\end{subfigure}
\begin{subfigure}
 \centering
 \includegraphics[width=.32\linewidth]{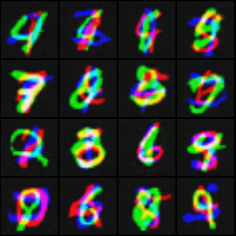}  
\end{subfigure}
\caption{Sampled images on the stacked MNIST dataset. {\bf Left}: Ground-truth. {\bf Middle}: LSGAN. {\bf Right}: IVLSGAN. Images generated by our method are more diverse.}
\label{stackMNISTimage}
\vspace{0.3cm}
\end{figure}

\paragraph{Ablation study}
Our ablation study is conducted on the CIFAR-10 dataset. First, we show the effectiveness of the intervention loss. We consider two cases, IVLSGAN without the intervention loss (achieved by setting $\mu=0$), and standard IVLSGAN (here $\mu$ is set to be 0.5). From Figure \ref{ablation_IVloss} we can find that the intervention loss makes the training process much smoother and leads to a lower FID score in the end. 

\begin{figure}
  \begin{minipage}[c]{0.47\linewidth}
    \vspace{0.05cm}
    \centering
    \includegraphics[width = \linewidth]{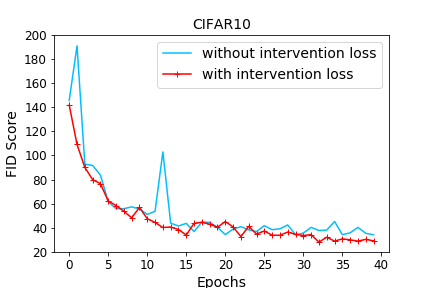}
    \figcaption{Training curve of IVLSGAN, with and without the intervention loss.}
    \label{ablation_IVloss}
  \end{minipage}
  \hfill
  \begin{minipage}[c]{0.5\linewidth}
    \begin{center}
    \vspace{0.05cm}
    \tabcaption{Minimum FID scores of IVLSGAN under different hyperparameter settings on the CIFAR10 dataset, calculated every 10 epochs.}
    \label{ablation_table}
    \vspace{0.2cm}
    \begin{tabular}{llc}
    \toprule
      & & FID score \\
    \midrule
    $\mu=0.5$&$k=2$ & 29.2 \\
    $\mu=0.5$&$k=4$ & 28.2\\
    $\mu=0.5$&$k=10$ & 41.2 \\
    $\mu=0.5$&$k=20$ & 36.1\\
    \midrule
    $\mu=0$ & $k=4$ & 34.5 \\
    $\mu=0.25$ & $k=4$ & 29.6\\
    $\mu=0.5$&$ k=4$ & 28.2 \\
    $\mu=1$&$k=4$ & 39.7\\
    \bottomrule
    \end{tabular}
    \end{center}
  \end{minipage}
\vspace{0.1cm}
\end{figure}
We also investigate the performance of our model using different number of blocks for the block substitution interventions and different regularization coefficients for the intervention loss.
The results are presented in Table \ref{ablation_table}. It can be noticed that to some extent our models' performance is not sensitive to the choice of hyperparameters and performs well under several different hyperparameter settings. However, when the number of blocks or the scale of IV loss becomes too large the performance of our model gets worse.

\section{Conclusion}
We have presented a novel model, intervention GAN (IVGAN), to stabilize the training process of GAN and alleviate the mode collapse problem.
By introducing auxiliary Gaussian invariant interventions to the latent space of real images and feeding these perturbed latent representations into the generator, we create intermediate distributions that interpolate between the generated distribution of GAN and the data distribution.
The intervention loss based on these auxiliary intervened distributions, together with the reconstruction loss, are added as regularizers to the objective to provide more informative gradients for the generator, significantly improving GAN's training stability and alleviating the mode collapse problem as well.

We have conducted a detailed theoretical analysis of our proposed approach, and illustrated the advantage of the proposed intervention loss on a toy example.
Experiments on real-world datasets as well as the stacked MNIST dataset demonstrate that, compared to the baseline methods, IVGAN variants are stabler and smoother during training, and are able to generate images of higher quality (achieving state-of-the-art FID scores) and diversity.

\small{

\bibliography{main}
}

\end{document}